\documentclass[letterpaper,11pt]{article}
\usepackage[top=1.25in, bottom=1.25in, left=1.20in, right=1.20in]{geometry}
\pdfoutput=1

\usepackage{newtxtext}
\RequirePackage[OT1]{fontenc}
\RequirePackage[numbers]{natbib}
\RequirePackage{amsthm,amsmath}
\usepackage{hyperref}
\usepackage{url}

\usepackage{dsfont} 
\usepackage{xspace}
\usepackage[usenames,dvipsnames]{xcolor} 

\usepackage{graphicx}

\usepackage{algorithm}
\usepackage{algorithmic}

\usepackage{multirow}
\usepackage{hhline}

\usepackage{booktabs}

\usepackage{capt-of}

\newtheorem{theorem}{Theorem}
\newtheorem{lemma}[theorem]{Lemma}
\newtheorem{corollary}{Corollary}
\newtheorem{defn}{Definition}

\def\eps{\ensuremath{\epsilon}\xspace}

\allowdisplaybreaks 

\def\one{\ensuremath{\mathds{1}\xspace}}

\def\rarrow{\ensuremath{\rightarrow}\xspace}

\def\sig{\ensuremath{\sigma}\xspace}

\def\eps{\ensuremath{\epsilon}\xspace}

\def\dt{{\ensuremath{\delta}\xspace} }

\def\sig{\ensuremath{\sigma}\xspace}

\usepackage[export]{adjustbox}

\usepackage{pbox} 

\newcommand{\fr}[2]{ { \frac{#1}{#2} }}
\def\lt{\left}
\def\rt{\right}

\makeatletter
\newcommand{\vast}{\bBigg@{3}}
\newcommand{\Vast}{\bBigg@{4}}
\makeatother

\def\w{{{\boldsymbol w}}}
\def\x{{{\boldsymbol x}}}
\def\y{{{\boldsymbol y}}}

\def\la{\langle}
\def\ra{\rangle}

\def\lam{\ensuremath{\lambda}}

\usepackage{pifont}
\def\p{{{\boldsymbol p}}}
 
\def\cD{\ensuremath{\mathcal{D}}\xspace}

\def\w{{{\boldsymbol w}}}

\def\cA{\ensuremath{\mathcal{A}}\xspace} 
 
\def\cM{\ensuremath{\mathcal{M}}\xspace}

\newcommand{\blue}[1]{{\color[rgb]{.3,.5,1}#1}}

\usepackage{enumitem}

\def\g{\ensuremath{\bg}}

\DeclareMathOperator{\Regret}{Regret}
\DeclareMathOperator{\Wealth}{Wealth}

\def\lam{{\ensuremath{\lambda}\xspace} }

\def\a{\ensuremath{\boldsymbol{a}}}

\def\R{\ensuremath{\mathbb{R}}} 
 
\def\E{\ensuremath{\mathbb{E}}}

\def\cF{{\ensuremath{\mathcal{F}}}}

\def\Om{{\ensuremath{\Omega}}} 

\DeclareMathOperator{\EE}{\mathds{E}}

\DeclareMathOperator{\PP}{\mathds{P}}

\def\hg{\ensuremath{\hat{g}}} 
\def\nab{\ensuremath{\nabla}} 

\def\bfone{{{\mathbf 1}}}

\newcommand*\diff{\mathop{}\!\mathrm{d}}

\usepackage{prettyref}

\let\SS\undefined
\def\ddefloop#1{\ifx\ddefloop#1\else\ddef{#1}\expandafter\ddefloop\fi}
\def\ddef#1{\expandafter\def\csname c#1\endcsname{\ensuremath{\mathcal{#1}}}}
\ddefloop ABCDEFGHIJKLMNOPQRSTUVWXYZ\ddefloop
\def\ddef#1{\expandafter\def\csname #1#1\endcsname{\ensuremath{\mathbb{#1}}}}
\ddefloop ABCDEFGHIJKLMNOPQRSTUVWXYZ\ddefloop

\def\ddef#1{\expandafter\def\csname b#1\endcsname{\ensuremath{\boldsymbol{#1}}}}
\ddefloop abcdeghijklmnopqrstuvwxyz\ddefloop

\newcommand{\bxi}{\boldsymbol{\xi}}

\let\SS\undefined
\def\SS{{\mathbb{S}}}
\let\EE\undefined

\let\PP\undefined
\DeclareMathOperator{\EE}{\mathbb{E}}

\DeclareMathOperator{\PP}{\mathbb{P}}

\DeclareMathOperator{\erf}{\text{erf}}

\def\oned{{\normalfont{\text{1D}}}}

\def\hbg{\hat{\bg}}
\def\tM{\text{\normalfont{M}}}
\def\tD{\text{\normalfont{D}}}

\def\balpha{{\boldsymbol{\alpha}}}
\def\bbeta{{\boldsymbol{\beta}}}
\def\Lin{{\normalfont{\text{Lin}}}}

\def\hz{{\ensuremath{\hat{z}}}} 

\setcitestyle{authoryear,round,citesep={;},aysep={,},yysep={;}}
\ifdefined\nohyperref\else\ifdefined\hypersetup
\definecolor{mydarkblue}{rgb}{0,0.08,0.45}
\hypersetup{ %
  pdftitle={},
  pdfauthor={},
  pdfsubject={},
  pdfkeywords={},
  pdfborder=0 0 0,
  pdfpagemode=UseNone,
  colorlinks=true,
  linkcolor=mydarkblue,
  citecolor=mydarkblue,
  filecolor=mydarkblue,
  urlcolor=mydarkblue,
  pdfview=FitH}

\ifdefined\isaccepted \else
\hypersetup{pdfauthor={Anonymous Submission}}
\fi
\fi\fi

\usepackage[export]{adjustbox}

\usepackage[T1]{fontenc}
\usepackage[utf8]{inputenc}
\usepackage{lmodern}
\usepackage{anyfontsize}

\usepackage{booktabs}       
\usepackage{amsfonts}       
\usepackage{nicefrac}       
\usepackage{microtype}      

\usepackage{mathtools}
\usepackage{bbm}
\usepackage{amsfonts}

\usepackage{MnSymbol} 
\let\SS\undefined
\def\SS{{\mathbb{S}}} 
\def\bv{\ensuremath{\boldsymbol{v}}}

\mathtoolsset{showonlyrefs}

\usepackage{enumitem}

\title{Parameter-Free Online Convex Optimization\\with Sub-Exponential Noise}
\author{
    Kwang-Sung Jun \\
    Boston University \\
    \texttt{kjun@bu.edu}\\
    \and
    Francesco Orabona\\
    Boston University \\
    \texttt{francesco@orabona.com}\\
}
\date{}

\begin{document}

\maketitle

\defcitealias{cutkosky17online}{Cutkosky and Boahen (COLT 2017)}

\begin{abstract}%
We consider the problem of unconstrained online convex optimization (OCO) with sub-exponential noise, a strictly more general problem than the standard OCO.
In this setting, the learner receives a subgradient of the loss functions corrupted by sub-exponential noise and strives to achieve optimal regret guarantee, without knowledge of the competitor norm, i.e., in a parameter-free way. 
Recently, \citetalias{cutkosky17online} proved that, given unbounded subgradients, it is impossible to guarantee a sublinear regret due to an exponential penalty. 
This paper shows that it is possible to go around the lower bound by allowing the observed subgradients to be unbounded via stochastic noise. 
However, the presence of unbounded noise in unconstrained OCO is challenging; existing algorithms do not provide near-optimal regret bounds or fail to have a guarantee.
So, we design a novel parameter-free OCO algorithm for Banach space, which we call BANCO, via a reduction to betting on noisy coins. 
We show that BANCO achieves the optimal regret rate in our problem.
Finally, we show the application of our results to obtain a parameter-free locally private stochastic subgradient descent algorithm, and the connection to the law of iterated logarithms.
\end{abstract}

\section{Introduction}

In this paper, we are interested in the problem of unconstrained Online Convex Optimization (OCO) with sub-exponential noise.
In the standard unconstrained OCO problem, at each round $t$, an algorithm chooses an iterate $\bw_t \in \R^d$ and then receives a \textit{negative} subgradient $\bg_t \in -\partial \ell_t(\bw_t)$ of a convex loss function $\ell_t(\bx)$ given by an adversary.\footnote{
  The notation $\bg_t$ is a mnemonic for ``gain'' since the subgradients correspond to losses in online linear games.
}
The goal of the learner is to minimize the regret defined by the difference between the cumulative loss of the learner and that of the unknown, arbitrary comparator $\bu$: 
\begin{align*}
\Regret_T(\bu) = \sum_{t=1}^T \ell_t(\w_t) - \sum_{t=1}^T \ell_t(\bu)~.
\end{align*}
Departing from the standard setup, we consider a game where the learner receives a \textit{noisy} version $\hbg_t$ of $\bg_t$.
Specifically, we assume that the noise $\hbg_t - \bg_t$ is sub-exponential.
Note that such a setting nicely mirrors the one of optimization of a fixed convex function with a stochastic first-order oracle.

The presence of noise implies that $\bw_t$, a function of the past noisy subgradients, is also stochastic. Thus, it is natural to minimize the \textit{expected} regret:
\begin{align}
\label{eq:def-regret}
  \E [\Regret_T(\bu)] 
  = \E \lt[ \sum_{t=1}^T \ell_t(\bw_t) - \sum_{t=1}^T \ell_t(\bu)\rt]~.
\end{align}
We will define more formally the setting and noise in Section~\ref{sec:setting}.
Our goal is to achieve expected regret bounds that have optimal dependency on $\|\bu\|$ and $T$, that is the so-called \emph{parameter-free} or \emph{adaptive} OCO algorithms~\citep{FosterRS15,orabona16coin,FosterKMS17,cutkosky17online,Kotlowski17,cutkosky18blackbox,foster18online}.

Our problem is motivated by a recent lower bound result on the unconstrained OCO showing that, without prior information on the largest subgradient, parameter-free algorithms are doomed to suffer an exponential penalty $\exp(\max_t L_t/L_{t-1})$, where $L_t$ is dual norm of the largest subgradient up to time $t$~\citep{cutkosky17online}.
Given such a catastrophic negative result that implies the excessive power of the adversary, one may ask the following question: under what condition on the game can the learner minimize regret efficiently with unbounded subgradients?
Our study provides a positive answer by allowing subgradients observed by the learner to be unbounded via \textit{stochasticity}, which limits the adversarial power without restricting observed subgradients to be bounded.

In order to develop low-regret algorithms for noisy OCO, it is tempting to directly use existing algorithms and their guarantees.
However, these attempts either result in a suboptimal dependence on $\|\bu\|$ in the regret, namely $\|\bu\|^2$, or do not lead to nontrivial regret bounds (see Section~\ref{sec:warmup} for details).
This motivates the following question: does there exist an unconstrained noisy OCO algorithm whose expected regret scales as optimally with $\|\bu\|$ and $T$?
We answer this question in the affirmative by proposing a new Betting Algorithm for Noisy COins (BANCO).
BANCO enjoys expected regret
\[
O\left(\|\bu\| \sqrt{(G^2 + \sig^2) T\log(1+\|\bu\|T)}\right)
\]
in a smooth Banach space, where $G$ is the bound on the \textit{expected} negative subgradients $\bg_t$ and $\sig^2$ is the variance of the noisy negative subgradients $\hbg_t$.
Our result reveals that, despite the noisy and unbounded nature of the feedback, it is possible to adapt to the unknown and best-in-hindsight comparator just as in the noise-free environments, in expectation.

BANCO is constructed via a natural extension of the coin betting framework~\citep{orabona16coin}, where we reduce noisy OCO to a 1-d game of betting money on \textit{noisy} coin flips to maximize one's expected wealth.
The noisy OCO in Banach space is then reduced to the 1-d coin betting, equipped with any \textit{constrained} noisy OCO learner in a black-box manner.
We describe the coin betting view and its extension to Banach OCO in Section~\ref{sec:1d_oco} and~\ref{sec:banach} respectively.
Furthermore, we further show that the dependence on the variance $\sigma^2$ cannot be improved, also matching the dependence on $\|\bu\|$ up to logarithmic factors.
We stress that, combining our lower bound and the existing ones in the literature, our regret upper bound is unimprovable.
We discuss details on lower bounds in Section~\ref{sec:lb}.

Finally, in Section~\ref{sec:apps}, we show some consequences of our results. Indeed, the noisy OCO problem and its algorithms have numerous applications as learning with noisy observations is a dominating paradigm of machine learning.
First, we show that our noisy OCO algorithm can be directly used for locally differentially-private stochastic subgradient descent (SGD). In fact, in private SGD noise is added on the subgradients to guarantee privacy, perfectly fitting our framework. In particular, we achieve the first parameter-free locally private SGD algorithm.
Second, we show that our algorithmic construction reveals a tight connection to concentration inequalities.
Specifically, we show that our algorithm implies a Banach valued concentration inequality that matches the rate of the law of the iterated logarithm.
The connection is made through a simple observation that a noisy coin betting potential directly implies a supermartingale, which is then combined with Doob's inequality to show concentration inequalities that hold for any time step $t$.

We conclude our paper with open problems in Section~\ref{sec:conclusion}.

\section{Problem Definition and Preliminaries} 
\label{sec:setting}

In this section, we describe our notations, formally define the problem, and provide background on coin betting.

\paragraph{Notations.}
The dual of a Banach space $V$ over a field $F$, denoted by $V^\star$, is the set of all continuous linear maps $V\rarrow F$.
We use the notation $\la \bv, \w\ra$ to indicated the application of a dual vector $\bv\in V^\star$ to a vector $\w\in V$. 
$V^\star$ is also a Banach space with the dual norm: $\|\bv\|_\star = \sup_{\w\in V, \|\w\|\le1} \la \bv,\w\ra$.
We abbreviate $\x_1, \ldots, \x_t$ by $\x_{1:t}$.

\vspace{-.3em}
\paragraph{Online convex optimization with noise.}
In OCO with noise, as introduced in the introduction, the learner receives a noisy version $\hbg_t \in V^\star$ of the negative subgradient $\bg_t \in V^\star$.
Since the learner's predictions $\w_t \in V$ are a function of past noisy subgradients, the regret is also stochastic.
Therefore, our goal is the minimize the \textit{expected} regret defined in~\eqref{eq:def-regret}.

We assume that the true subgradients are bounded by $G$: $\|\bg_t\|_* \le G$.
Furthermore, the noise $\bxi_t := \hbg_t - \EE_t[\hbg_t]$ is conditionally zero-mean and has conditional finite variance measured with the dual norm:
\vspace{-.5em}
\begin{align}
\label{eq:bounded_variance}
\EE\lt[\|\bxi_t\|_*^2 \mid \bxi_{1:t-1}\rt] \le \sig^2, \forall t,
\end{align}
for some $\sig>0$.
Hereafter, we use the notation $\EE_t$ to denote $\EE [ \cdot \mid \bxi_{1:t-1}]$.
We also assume a tail condition such that $\bxi_t$ is conditionally sub-exponential with parameters $(\sigma_\oned^2, b)$:\footnote{$\beta$ is often qualified as $|\beta|<\fr1b$ in the literature. Our qualification is merely for ease of exposition.  }
\vspace{-.5em}
\begin{align}
\label{eq:sub_gamma}
\max_{\a: \|\a\|\le1} \ \E_t\left[\exp(\beta
\la \bxi_t, \ba \ra) \right] 
\leq \exp\left(\frac{\beta^2 \sigma_{\oned}^2}{2}\right), \  \forall |\beta| \leq \frac1b~.
\end{align}
One can show that, when \eqref{eq:sub_gamma} is achieved with equality, we have $\sigma_{\oned}^2 \le \sigma^2$.
The intuition of the condition above is that the tail of the noise $\bxi_t$ behaves well in any direction; a similar form of condition for sub-Gaussian vectors can be found in~\cite{hsu12tail}.
This noise definition covers a wide range of distributions, including Gaussian and Laplace. 
Consider the L2 norm for simplicity.
If $d=1$, we have $\sig^2 = \sig^2_\oned$.
This is not true in general and the relationship depends on the noise distribution and the norm being considered.
If $\bxi_t \sim \cN(0,s^2 \mathbf{I})$, then one can see that $\sig_\oned^2 = s^2$ and $\sig^2 = d s^2$.
As another example, the Laplace mechanism noise used in differentially-private learning satisfies the tail condition above; see Section~\ref{sec:private_sgd}.

\vspace{-.3em}
\paragraph{OCO as betting on noisy coins.}
One recent framework for unconstrained OCO is coin betting, which views the OCO game as maximizing a gambler's wealth via repeated betting on \textit{adversarial} coin flips~\citep{mcmahan13minimax,orabona16coin}. This framework provides a straightforward way to design algorithms that achieve optimal regret bounds with respect to any competitor, without imposing a bounded set for the competitor nor any parameter to tune, i.e., parameter-free.
Consider 1d OCO with $G=1$ for simplicity.
The gambler starts with the initial endowment $\Wealth_0 = \tau$ for some $\tau>0$.
In each iteration $t$, the gambler determines how much money to bet and whether to bet on heads ($+1$) or tails ($-1$), which is encoded as $|w_t|$ and $\text{sign}(w_t)$ respectively.
After the adversary's (continuous) coin outcome $g_t \in [-1,1]$ is revealed, the gambler's wealth, denoted by $\Wealth_t$, is updated additively:
$\Wealth_t = \Wealth_{t-1} + g_t w_t$.
That is, the gambler makes (loses) money when she gets the coin side correct (incorrect), and the amount of return (loss) is determined by $|g_t w_t|$ (respectively).
Developing successful strategies critically rely on designing a potential function $F_t(x)$ and an appropriate betting amount $w_t$ such that
\vspace{-.5em}
\begin{align}\label{eq:coinbetting}
  \Wealth_0 = F_0(0) \quad\text{ and }\quad   F_{t-1}(x) + g_t w_t \ge F_t(x + g_t), \forall t~.
\end{align}
One can show that the two properties above imply $\Wealth_t \ge F_t(\sum_{s=1}^t g_s)$ (the derivation is similar to~\eqref{eq:wealth-lb-proof} below).
\citet[Theorem 1]{mcmahan14unconstrained} show that a lower bound on $\Wealth_T$ is equivalent to an upper bound on the linearized regret w.r.t. a comparator $u$, $\sum_{t=1}^T g_t (u - w_t)$, which reveals a tight connection between coin betting and OCO.

In this paper, we extend the coin betting problem to noisy coin outcomes.
Specifically, the gambler observes a noisy version of the coin outcome $\hg_t \in \RR$ rather than $g_t = \EE[\hg_t]$.
While the extension appears obvious, the existing coin betting strategies \citep[e.g.][]{orabona16coin,orabona17training} cannot be applied to the noisy setting; their design ensures that the wealth never goes below 0 w.p. 1, which cannot be true for our setting as the coin outcome can be arbitrarily bad.

To cope with noisy coins, we develop a \textit{noisy coin betting} framework.
The key idea is that, although we cannot guarantee the nonnegativity of wealth, we can guarantee it for the \textit{expected} wealth.
Departing from the conditions for noise-free coin betting~\eqref{eq:coinbetting}, we assume that $F_t$ and $w_t$ satisfy the betting relationship in conditional expectation
\vspace{-.5em}
\begin{align}
\label{eq:coinbetting-condition}
F_{t-1}(x) + g_t w_t  \geq \E_t[F_t(x+\hg_t)]~.
\end{align}
This immediately implies that
\vspace{-.5em}
\begin{align}\label{eq:expected-wealth-bound}
  \E \Wealth_t 
  \ge \E \left[ F_t\left(\sum_{s=1}^t \hg_s\right) \right]~.
\end{align}
In fact, by induction, assume that~\eqref{eq:expected-wealth-bound} holds for $t-1$.
Then,
\vspace{-.5em}
\begin{align}\label{eq:wealth-lb-proof}
\begin{split}
\E \Wealth_t 
&= \E[ \Wealth_{t-1} + \E_t \hat g_t w_t ] 
\stackrel{(a)}{\ge} \E \left[ F_{t-1}\left(\sum_{s=1}^{t-1} \hg_s\right) + \E_t \hat g_t w_t \right] 
\\&= \E \left[ F_{t-1}\left(\sum_{s=1}^{t-1} \hg_s\right) + g_t w_t \right]
\stackrel{(b)}{\ge} \E \left[ F_t\left(\sum_{s=1}^t \hg_s\right) \right],
\end{split}
\end{align}
where $(a)$ is by the inductive hypothesis and $(b)$ is by~\eqref{eq:coinbetting-condition}.

\vspace{-.5em}
\section{The Devil is in the Details: Failing Approaches}
\label{sec:warmup}

As a warm-up, we discuss how one might attempt to extend existing algorithms for the noisy setting and why these approaches would fail.
For simplicity, consider that $V=\R^d$, the norm is the L2 norm, and $G=1$.
For this, we need algorithms that enjoy regret bounds \textit{without requiring a subgradient bound as an input}.
For example, one can apply online subgradient descent (OGD), which guarantees a regret bound w.r.t. the noisy subgradients:
\vspace{-.5em}
\begin{align*}
  \hat R^{\Lin}_T(\bu) := \sum_{t=1}^T \la\hbg_t, \bu - \w_t \ra = \fr{\|\bu\|^2}{2\eta} + \frac{\eta}{2} \sum_{t=1}^T \|\hbg_t\|^2~.
\end{align*}
Notice that $\hat R^\Lin_T(\bu)$ itself does not bound $\Regret_T(\bu)$ and one must turn to either \textit{expected} or \textit{high probability} regret bounds.
With the choice of the step size $\eta = 1/\sqrt{(\sig^2+1)T}$, we have an expected regret bound:
\vspace{-.5em}
\begin{align*}
\EE [\Regret_T(\bu) ]
\stackrel{(a)}{\le} \EE \lt[\sum_{t=1}^T \la\g_t, \bu - \bw_t \ra \rt]
\stackrel{(b)}{=} \EE \lt[\sum_{t=1}^T \la\hbg_t, \bu - \bw_t \ra \rt]
= O\lt((\|\bu\|^2 + 1)\sqrt{(\sig^2+1)T}\rt),
\end{align*}
where $(a)$ is by convexity and $(b)$ is by the tower rule.
However, the dependence on the unknown comparator $\bu$ is $\|\bu\|^2$, which is much larger than the best known rate, which is $\|\bu\|\sqrt{\log(1+\|\bu\|)}$ \citep{mcmahan14unconstrained}.
While there exist algorithms that almost achieve this rate w.r.t. $\|\bu\|$ without requiring a bound on $\hbg_{1:T}$ as input (e.g., \cite{cutkosky17online} with $\gamma \approx \fr{1}{2}$), the lower bound of~\cite{cutkosky17online} implies that the overall regret bound cannot be sublinear.

Another attempt is to leverage the fact that the noisy subgradients are bounded with high probability.
Consider for example a 1d OCO problem with $(\sig_\oned^2, 0)$-sub-exponential noise in which case $\sig=\sig_\oned$.
Let $E_1$ be the event that $|\hg_t| \le g_t + \sig\sqrt{\log(T/\dt)}$ for all $t\le T$ (omitting constants), which satisfies $\PP(\neg E_1) \le \dt$.
Using the standard parameter-free OCO algorithms such as the one in \cite{mcmahan14unconstrained}, one may obtain the following bound under the event $E_1$:
\vspace{-.5em}
\begin{align}\label{eq:hpp-bound}
  \hat R^{\Lin}_T(u) = O\lt(|u|\lt(G + \sig\sqrt{\log(T/\dt)}\rt)\sqrt{ T \log(1 + |u|)}\rt),
\end{align}
which is, again, \textit{not} an upper bound on $\Regret_T(u)$, not even under $E_1$.\footnote{
  One may attempt to derive a high probability regret bound via a decomposition $\sum_{t=1}^T \la g_t, u-w_t \ra = \sum_{t=1}^T \la g_t - \hg_t, u-w_t \ra + \sum_{t=1}^T \la \hg_t, u-w_t \ra$. However, the first summation involves $w_t$ that is unbounded, and analyzing the behavior of $w_t$ appears nontrivial. We leave high probability regret bounds as future work.
}
Define the linearized regret: $R_T^{\Lin}(u) = \sum_{t=1}^T \la\g_t, \bu - \bw_t \ra$.
In a special case where there exists $c>0$ such that $\hat R_T^\Lin(u) \ra \le c|u|T$ (though we explain below this is unrealistic), one may have an expected regret bound as follows:
\vspace{-.5em}
\begin{align*}
\EE \Regret_T(u) 
&\le \EE [R^\Lin_T(u)]
= \EE [\hat R^\Lin_T(u)] 
\\&= \EE [\hat R^\Lin_T(u) | E_1] \cdot \PP(E_1) + \EE [ \hat R^\Lin_T(u) | \neg E_1) \cdot \PP(\neg E_1) \\
&= O\lt(|u|\lt(G + \sig\sqrt{\log(T/\dt)}\rt)\sqrt{ T \log(1 + |u|)}\rt) + c|u|T\dt~.
\end{align*}
Indeed, the assumption $\hat R^\Lin_T(u) \le c|u|T$ would be true for \textit{constrained} OCO with \textit{bounded} noise $\xi_t$.
However, our case is neither constrained nor with bounded noise.
For a fixed $u$, if $u - w_T>0$, then $\hg_T$ can be arbitrarily large, making the regret much larger than $c|u|T$ for any $c$.
Such an issue caused by unbounded noise poses a significant challenge in designing unconstrained algorithms adapting to the unknown comparator $\bu$ under noisy feedback.

Finally, we remark that, for linear losses, the standard OGD can have an expected regret that does not scale with $\sig$.
This, however, does not generalize to generic convex losses.
In fact, our lower bound result in Section~\ref{sec:lb} shows that the factor $\sig$ in the expected regret bound cannot be avoided in general.
We elaborate more on this in Appendix~\ref{sec:ogd_olo}.

\vspace{-.5em}
\section{One-dimensional Betting Algorithm with Noisy Coins}
\label{sec:1d_oco}

In this section, we show how to construct noisy coin betting potentials. 
We focus on potential functions $F_t$ and associated betting strategy $w_t$ defined as follows:
\begin{align*}
F_t(x) = \int f_t(x, \beta) \diff \pi(\beta), \quad\text{and}\quad w_t = \int \beta f_{t-1}\left(\sum_{i=1}^{t-1} \hat{g}_i, \beta\right) \diff \pi(\beta),
\end{align*}
for some functions $f_t(x,\beta)$, and a prior $\pi(\beta)$.
This defines a family of noisy coin betting potentials, parameterized by the prior $\pi$. 
While this kind of potentials have been used by \citet{chernov10prediction,koolenVE15} for parameter-free algorithms for learning with expert advice, our key novelty lies in blending the effect of sub-exponential noise into the potential naturally, making it amenable to analysis. 

Our construction is based on the following key inequality for sub-exponential random variables.
\begin{lemma}
Let $\hg$ be a $(\sigma^2,b)$-sub-exponential random variable, with mean $g$ such that $|g| \le G$. Let $k_1$ satisfy 
\begin{align}\label{eq:def-k1}
1 - k_1 = \exp(-k_1 - k_1^2),
\end{align}
that is $k_1= 0.683803\dots$. Then, for any $\beta$ such that $|\beta|\leq \min(k_1/G,1/b)$, we have
\begin{equation}
\label{eq:noisy-subg}
1 + \beta\E_{\hg} [\hg] 
\ge \E_{\hg} \exp\left(\beta \hg - \beta^2 \left(\frac{\sig^2}{2} + G^2\right)\right)~.
\end{equation}
\end{lemma}
\begin{proof}
Given that $|\beta| \leq \min(k_1/G,1/b)$, we have $\beta g \ge -k_1$ and $1+\beta g \ge e^{\beta g - \beta^2 g^2}$.
Then,
\begin{align*}
  1 + \beta\E_{\hg} [\hg] 
  &= 1 + \beta  g
  \ge \exp(\beta g - \beta^2 g^2) 
  \ge \E_{\hg} \exp\left(\beta \hg - \beta^2 \left(\frac{\sig^2}{2} + g^2\right)\right) \\
  &\ge \E_{\hg} \exp\left(\beta \hg - \beta^2 \left(\frac{\sig^2}{2} + G^2\right)\right),
\end{align*}
where the second inequality is due to $\E \exp(\beta(\hg - g)) \le \exp(\beta^2\sig^2/2)$ for all $|\beta|\leq \frac{1}{b}$.
\end{proof}
From this lemma, multiplying the right hand side of the equation for $i=1$ to $t$, it is natural to define our noisy coin betting potential as
\begin{align}
\label{eq:stochastic_coin_betting_potential}
F_t(x) := \tau \int_{-a}^a \exp\left(\beta x - \beta^2 t \left(\frac{\sig^2}{2} + G^2\right) \right)  \diff \pi(\beta),
\end{align}
and associated prediction strategy
\[
w_t = \tau \int_{-a}^{a} \beta\exp\left(\beta \sum_{s=1}^{t-1} \hg_s - \beta^2 (t-1) \left(\frac{\sig^2}{2} + G^2\right) \right)  \diff \pi(\beta),
\]
where $a \le \min(k_1/G,1/b)$ and $\pi$ has support in $[-a,a]$.
In this way, we obtain our Betting Algorithm for Noisy COins (BANCO) and summarize it in Algorithm~\ref{alg:oned}.
In the following theorem we show that \eqref{eq:stochastic_coin_betting_potential} satisfy our assumptions.

\begin{algorithm}[t]
  \begin{algorithmic}
    \STATE \textbf{Require:} sub-exponential parameters $(\sig^2,b)$, expected subgradient bound $G$, initial money $\tau$.
    \FOR {$t=1$ \TO $T$}
    \STATE Play $w_t = \tau \int_{-a}^{a} \beta\exp\left(\beta \sum_{s=1}^{t-1} \hg_s - \beta^2 (t-1) \left(\frac{\sig^2}{2} + G^2\right) \right)  \diff \pi(\beta)$ where $a = \min\left(\fr{k_1}{G}, \fr{1}{b}\right)$.
    \STATE Receive $\hg_t \in \RR$.
    \ENDFOR
  \end{algorithmic}  
  \caption{Betting Algorithm for Noisy COins (BANCO)}
  \label{alg:oned}
\end{algorithm}

\begin{theorem}
  Let $w_t$ be computed by Algorithm~\ref{alg:oned}.
  Then $F_t$ in \eqref{eq:stochastic_coin_betting_potential} is a noisy betting potential. 
\end{theorem}
\begin{proof}
From the definition it is obvious that $F_0(0)=\tau$. We then have to show that
$
\E_t F_t(x+\hat{g}_t) \le F_{t-1}(x) + g_t w_t
$.
Hence, consider
\begin{align*}
\E_t F_t(x+\hat g_t) 
&=\tau \E_t \int_{-a}^a \exp\left(\beta (x+\hg_t) - \beta^2 t\left(\sig^2/2 + G^2\right)\right) \diff\pi(\beta) \\
&= \tau \E_t \int_{-a}^a \exp\left(\beta\hg_t - \beta^2\left(\sig^2/2 +G^2\right) \right) \exp\left(\beta x - \beta^2 (t-1)\left(\sig^2/2 + G^2\right)\right) \diff\pi(\beta) \\
&\stackrel{(a)}{\le} \tau \int_{-a}^a (1 + \beta \E_t \hg_t) \exp\left(\beta x - \beta^2 (t-1)(\sig^2/2 + G^2)\right) \diff\pi(\beta) 
\stackrel{(b)}{=} F_{t-1}(x) + g_t w_t,
\end{align*}
where $(a)$ is due to~\eqref{eq:noisy-subg} and $(b)$ is by Fubini's theorem.
\end{proof}

In the standard coin betting, a lower bound on the wealth is equivalent to an upper bound on the regret for linearized losses by \citep[Theorem 1]{mcmahan14unconstrained}.
We extend this result to the \textit{expected} wealth and linearized regret, proof in Appendix~\ref{sec:proof_thm_reward-regret}.
\begin{theorem}[Reward-Regret relationship]
\label{thm:reward-regret}
Let $V,V^\star$ be a pair of dual vector spaces. Let $F:V^\star \to \R \cup \{+\infty\}$
be a proper convex lower semi-continuous function and let $F^\star:V \to \R \cup
\{+\infty\}$ be its Fenchel conjugate. 
Let $\tau \in \R$.
Consider the $\sigma$-algebra $\cF_t = \sig(\hat\bg_1, \ldots, \hat\bg_{t-1})$.
Let $\bw_t$ be $\cF_t$-measurable, $\forall t \in \{1,\ldots,T\}$.
Then,
\[
  \underbrace{\tau+\E\left[\sum_{t=1}^T \langle \hat{\bg}_t, \bw_t \rangle\right]}_{\E\lt[\Wealth_T\rt]} 
  \ge \E\left[F\left( \sum_{t=1}^T \hat{\bg}_t \right)\right]
  \quad \implies \quad
  \forall \bu \in V, ~
  \underbrace{\E\left[\sum_{t=1}^T \langle \hat{\bg}_t, \bu - \bw_t \rangle\right]}_{\E[\hat R^{\Lin}_T(\bu)]} \le F^\star(\bu) + \tau~.
\]
\end{theorem}
Hence, to obtain a regret bound from the above theorem, we just need to compute the Fenchel conjugate of the noisy coin betting potential $F_T$. 
We remark that in the standard non-noisy setting the reward-regret relationship holds for both directions (i.e., wealth bound iff regret bound) rather than one direction only. 
It remains unclear to us whether such a direction is true or not.

To construct a specific algorithm, it remains to choose the prior $\pi$.
While one can choose any prior, it is preferred to have a closed form expression for $w_t$.
We choose $\text{Uniform}[-a,a]$ for simplicity, which leads to the following closed form: with shorthands $x = \sum_{s=1}^{t-1} \hat g_s$ and $y= (t-1)(\sig^2/2 + G^2)$,
\begin{align*}
  w_t = \tau \fr{e^{-a(ay+x)}\lt( \sqrt\pi x  \exp\lt(\fr{(2ay+x)^2}{4y}\rt)       \lt( \erf(\fr{2ay+x}{2\sqrt y}) + \erf(\fr{2ay-x}{2\sqrt y}) \rt) + 2\sqrt y(1-e^{2ax}) \rt)}{8ay^{3/2}}~.
\end{align*}
Note that a similar prediction strategy was also proposed in \citet{koolenVE15}. It is easy to verify that another choice that results in a closed form update with an equivalent wealth guarantee is with a Gaussian prior centered at zero.
For improving numerical precision for computing $w_t$ above, we refer to~\citep{koolenblog}.

In the following theorem we calculate the Fenchel conjugate of of this potential function from which the regret bound immediately follows by Theorem~\ref{thm:reward-regret}, proof in Appendix~\ref{sec:proof_thm_conjugate_uniform_prior}.
\begin{theorem}
\label{thm:conjugate_uniform_prior}
Let $F(x) = \tau \int_{-a}^a \exp(\beta x - \beta^2 S) \diff \pi(\beta)$ where $\pi(\beta)$ is $\text{Uniform}[-a,a]$. 
Then,
\begin{align*}
F^\star(u) 
\le \max\left\{|u|\sqrt{2S\ln\left(1 + \tfrac{16ea^2S^2u^2}{\tau^2}\right)},  \frac{8}{3a} |u|\ln\left(\frac{32}{3 e a\tau}|u|\right) \right\}~.
\end{align*}
\end{theorem}
Applying the two theorems above with $S = T(\sig^2/2 + G^2) $ and $a = \min(k_1/G,1/b)$, where $k_1$ is defined in~\eqref{eq:def-k1}, we have the expected regret guarantee of BANCO:
\begin{align*}
  \E[\Regret_T(u)]
  \le \tau+|u|\max&\lt\{\sqrt{2\left(G^2+\frac{\sig^2}{2}\right)T\ln\left(1 + 16e\min\left(\frac{k_1}{G},\frac1b\right)^2 T^2 \left(G^2+\frac{\sig^2}{2}\right)^2 \frac{u^2}{\tau^2}\right)}, \rt.\\ 
  & \quad 
   \frac{8}{3}\max\left(\frac{G}{k_1},b\right)\ln\left(32\max\left(\frac{G}{k_1},b\right)\frac{|u|}{3 e\tau}\right) \Bigg\}~.
\end{align*}

\section{Banach Online Convex Optimization with Noise}
\label{sec:banach}

In this section, we extend the parameter-free algorithm, BANCO, to Banach spaces.
Attempting to extend the 1d algorithm to higher dimensional spaces would require an ad hoc analysis specialized to the particular algorithm.
Instead, we leverage a black-box reduction: we take any \textit{constrained} noisy OCO algorithm for Banach space and turn it into an \textit{unconstrained} one via BANCO.

Let $V$ be a Banach space and the negative subgradients $\hbg_t \in V^\star$ satisfy $ \| \EE_t \hbg_t\|_\star \le G$. 
Define $S$ to be the unit ball in $V$.
We summarize our reduction in Algorithm~\ref{alg:reduction}, which is a direct extension of~\cite{cutkosky18blackbox} for noisy subgradients.
The key feature of the algorithm is a black-box reduction that takes two learners: $(i)$ the 1d coin-betting that predicts the \textit{magnitude} $w_t \in \RR$ and $(ii)$ a $d$-dimensional learner $\cA_\tD$ that predicts the \textit{direction} $\by_t \in S$.
The reduction then makes the combined prediction by $\bx_t = w_t \y_t$
After receiving the noisy negative subgradient $\hbg_t$ evaluated at $\bx_t$, we feed $s_t = \la \hbg_t, \by_t\ra$ into the coin-betting algorithm and $\hbg_t$ into $\cA_\tD$ as the subgradient.

\begin{algorithm}[t]
  \begin{algorithmic}
    \STATE \textbf{Require:} Banach space $V$, learner $\cA_\tD$ with domain being the unit ball $S \subset V$
    \FOR {$t=1$ \TO $T$}
    \STATE Get point $w_t\in\RR$ from BANCO, Algorithm~\ref{alg:oned}
    \STATE Get point $\by_t\in S$ from $\cA_\tD$
    \STATE Play $\bx_t = w_t \by_t \in V$
    \STATE Receive a noisy negative subgradient $\hbg_t$ such that $\EE[\hbg_t] \in -\partial \ell_t(\bx_t)$
    \STATE Set $s_t = \la \hbg_t, \by_t \ra$
    \STATE Send $s_t$ to BANCO, Algorithm~\ref{alg:oned}
    \STATE Send $\hbg_t$ as the $t$-th negative subgradient to $\cA_{\tD}$
    \ENDFOR
  \end{algorithmic}  
  \caption{BANCO in Banach Spaces}
  \label{alg:reduction}
\end{algorithm}

Theorem~\ref{thm:banach-oco} below shows that the expected regret of Algorithm~\ref{alg:reduction} is nicely decomposed into two expected regrets, each from the noisy coin betting algorithm and $\cA_{\tD}$.
The fact that we require the expected regret of $\cA_\tD$ w.r.t. the \textit{unit norm comparator} frees us from tuning the parameter of $\cA_\tD$ for the optimal step size, delegating the burden of adaptation to the noisy coin betting algorithm. 
The proof is simple and immediate from \cite{cutkosky18blackbox}, but for completeness we report it in Appendix~\ref{sec:proof_thm_banach-oco}.
\begin{theorem}
\label{thm:banach-oco}
  Suppose $\cA_\tD$ obtains expected regret $R^{\tD}_T(\bu) := \sum_{t=1}^T \la\hbg_t, \bu - \y_t \ra$ for any competitor $\bu$ in the unit ball $S \subset V$ and the coin betting algorithm obtains expected regret $R^\tM_T(v) := \sum_{t=1}^T s_t\cdot(v - w_t)$ for any competitor $v\in\RR$.
  Then, Algorithm~\ref{alg:reduction} guarantees
  \begin{align*}
    \EE \Regret_T(\bu) \le R^\tM_T( \|\bu\| ) + \|\bu\| R^{\tD}_T(\bu/\|\bu\|),
  \end{align*}
  where we define $\bu/\|\bu\| = \boldsymbol{0}$ when $\bu = \boldsymbol{0}$.
\end{theorem}

Note that the loss $\la \hbg_t, \by_t\ra = \la \bg_t, \by_t \ra + \la \bxi_t, \by_t \ra$ fits the 1d noisy OCO setting exactly.
To see this, $|\la \y_t, \bg_t\ra| \le \| \bg_t \|_\star \le G$.
Furthermore, the random variable $\la \by_t, \bxi_t \ra \mid \bxi_{1:t-1}$ is $(\sig_\oned^2,b)$-sub-exponential since
\begin{align*}
\forall |\nu|\le 1/b,\quad  \EE_t \exp(\nu \la \bxi_t, \by_t \ra) \le \exp(\nu^2 \sig_\oned^2/2),
\end{align*} 
where we use the fact $\|\by_t\| \le 1$ and our noise assumption~\eqref{eq:sub_gamma}.

For $\cA_\tD$, one can invoke any algorithm for the Banach space of interest~\citep{srebro11universality}.
In particular, if $V$ is $(2,\lam)$-uniformly convex~\citep{pinelis15rosenthal}, we can use online mirror descent with stepsizes $\eta_t=\fr{\sqrt{\lam}}{\sqrt{\sum_{s=1}^t \|\hbg_s\|_\star^2}}$ and predictions projected onto the unit ball $S$.
One can then immediately obtain the expected regret bound with noisy subgradients: 
\begin{align*}
  R_T^{\tD}\lt(\fr{\bu}{\|\bu\|}\rt) 
  &\le \EE \lt[ \sum_{t=1}^T \lt\la \bg_t, \fr{\bu}{\|\bu\|} - \y_t\rt\ra \rt] 
  = \EE \lt[ \sum_{t=1}^T \lt\la \hbg_t, \fr{\bu}{\|\bu\|} - \y_t\rt\ra \rt] 
  = O\lt(\EE \lt[ \fr{1}{\sqrt{\lam}}\sqrt{\sum_{t=1}^T \|\hbg_t \|_\star^2} \rt] \rt) \\
  &\stackrel{(a)}{\le} O\lt(\fr{1}{\sqrt{\lam}}\sqrt{\sum_{t=1}^T\lt( \EE ||\bg_t||_\star^2 + \sig^2 \rt)}\rt),
\end{align*}
where $(a)$ uses Jensen's inequality and the fact that $\EE[ \|\hbg_t\|_\star^2] \le 2 \EE[\| \bg_t \|_\star^2] + 2 \sig^2$.

Finally, Algorithm~\ref{alg:reduction} equipped with the uniform prior in the noisy coin betting algorithm and $\cA_\tD$ chosen as above enjoys the following expected regret bound:
\begin{align*}
  \E \Regret_T(\bu) 
  &= O\Bigg( \|\bu\| \max\lt\{(G+b) \ln\fr{\|\bu\| (G+b) }{\tau}  , \sqrt{(G^2 + \sig_\oned^2)T \ln\lt(\fr{\|\bu\|(G^2 + \sigma_\oned^2) T }{\tau}  +1\rt) } \rt\}  \\
  &\quad+ \fr{\|\bu\|}{\sqrt{\lam}}\sqrt{\sum_{t=1}^T\lt( \E\|\bg_t\|_\star^2 + \sigma^2 \rt)} + \tau \Bigg)~.
\end{align*}
Examples of $(2,\lam)$-uniformly convex Banach space include Hilbert spaces with $2$-norm (in which case $\lam=1$), as well as with $p$-norm with $p\in(1,2]$ (in which case $\lam=p-1$).
The runtime of Algorithm~\ref{alg:reduction} is dominated by the direction learner $\cA_\tD$ since the runtime of BANCO does not scale with $d$.
In other words, the black-box reduction adds little computational overhead while adapting to the unknown best-in-hindsight comparator from noisy feedback.

\section{Lower bound}
\label{sec:lb}

In this section, we investigate lower bounds on the noisy OCO problem.
Theorem~\ref{thm:lb} shows that our dependence on the noise variance $\sig^2$ is unimprovable in general.
\begin{theorem}\label{thm:lb}
  Let $\sig\ge2$, $p \ge1$.
  Let $q$ satisfy $1/q = 1 - 1/p$.
  Denote by $\nab\hat \ell_t(\x)$ a noisy subgradient of $\ell_t(\x)$.
  For any algorithm, there exists a noisy OCO instance with 1-Lipschitz loss functions w.r.t. $p$-norm and $\EE ||\nab \hat \ell_t(\x) - \nab \ell_t(\x)||_q^2 \le \sig^2$ and a comparator $u$  s.t.
  \begin{align*}
  p \ge 2 &\implies \EE \Regret_T(u) \ge \min\lt\{ c_0 \sig ||\bu||_p d^{\fr{1}{2} - \fr{1}{p}}\sqrt{T} ,  \fr{1}{18} ||\bu||_p d^{-\fr{1}{p}} T \rt\} \quad \text{ and }
  \\p \in [1,2]    &\implies  \EE \Regret_T(u) \ge \min\lt\{  c_0 \sig ||\bu||_p \sqrt{T} ,  \fr{1}{18} ||\bu||_p T \rt\},
  \end{align*}
  where $c_0$ is a universal constant.
\end{theorem}
The main argument of the proof is based on a carefully constructed stochastic optimization instance, which is connected to online convex optimization through the online-to-batch conversion~\citep{littlestone89online}; see Appendix~\ref{sec:proof_thm_lb} for details.

Note that our lower bound's dependence on $\|\bu\|$ mismatches our upper bound by a factor of $\sqrt{\log(1+\|\bu\|)}$.
The reason is that the constructed problem class for the proof is an easier optimization problem where the learner knows the norm of the best competitor $\bu$.
One may attempt to extend the lower bound of~\cite{orabona13dimension} to the noisy setting, which has the right dependence on $\|\bu\|$.
However, their construction is based on linear losses in which there exists a learner whose expected regret does not scale with $\sig$, as we show in Appendix~\ref{sec:ogd_olo}.

Nevertheless, we claim that the expected regret of the noisy OCO is 
\begin{align*}
\Omega\lt(G\|\bu\| \sqrt{T \log(1+\|\bu\|)} + \sigma \|u\|\sqrt{T}\rt)~,
\end{align*}
which does include the extra logarithmic factor in $\|u\|$.
The claim is based on the lower bound $\Omega(G\|\bu\| \sqrt{T \log(1+\|\bu\|)})$ for noise-free unconstrained OCO~\cite[Theorem 2]{orabona13dimension}.
Specifically, suppose there exists an algorithm $\cA$ achieving a strictly better order of regret bound than $G\|\bu\| \sqrt{T \log(1+\|\bu\|)}$ in the noisy setting. 
We can then solve the standard noise-free problem by adding some infinitesimal noise to the observed (non-noisy) gradients by ourselves and feeding that noisy gradients to $\cA$.
This leads to a better regret bound than the lower bound for the noise-free problem, which is a contradiction.
\section{Applications}
\label{sec:apps}

We discuss two applications of our results to domains beyond the one of online learning.

\vspace{-.4em}
\subsection{Parameter-Free Locally Differentially Private SGD}
\label{sec:private_sgd}

In this section, we describe the application of our algorithm to the locally differentially private SGD~\citep{DuchiJW14,SongCS15}.
An $\epsilon$-differentially private algorithm must guarantee that the log-likelihood ratio of the outputs of the algorithm under two databases differing in a single individual's data is smaller than $\epsilon$~\citep{DworkMNS06}. In the stricter definition of local differential privacy~\citep{WassermanZ10,KasiviswanathanLNRS11,DuchiJW14,SongCS15} instead an untrusted algorithm is allowed to access a perturbed version of a sensitive dataset only through a sanitization interface. In particular, the sanitization mechanism must guarantee that the log-likelihood ratio of the data of two individuals $i$ and $j$ is smaller than $\epsilon$.

\vspace{-.2em}
\begin{defn}[Local Differential Privacy~\citep{DuchiJW14,SongCS15}]
\label{def:local_diff_privacy}
Let $\cD$ be the space of data points and $D=\{X_1, \dots, X_n\} \subseteq \cD$ be a sensitive dataset where each $X_i\sim \rho_X$ corresponds to data about individual $i$. A randomized sanitization mechanism $M$ which outputs a disguised version $\{U_1, \dots U_n\}$ of $D$ is said to provide $\epsilon$-local differential privacy to individual $i$, if, for every event $S$, 
\[
\sup_{x,x' \in \cD} \frac{\PP[U_i\in S|X_i=x]}{\PP[U_i \in S| X_i = x']} \leq \exp(\epsilon),
\]
where the probability is w.r.t. the randomization in the sanitization mechanism.
\end{defn}
\vspace{-.2em}

The local differential setting can be specialized to SGD~\citep{SongCS15}.
Consider the minimization of function $H(\bw)=\E_{\bx \sim \rho_X}[h(\bw,\bx)]$, where $h(\bw,\bx)$ is convex in the first argument and $\bx$ represents sensitive data about one individual. The sanitization mechanism becomes the noisy subgradient oracle that returns $\mathcal{G}(\bw) \in\partial h(\bw, \bx) + \bxi_t$ when queried on $\bw$, where $\bx$ is coming i.i.d. from $\rho_X$ and the noise $\bxi_t$ guarantees the local differential privacy~\citep{SongCS15}.

We now apply the results from Section~\ref{sec:banach}, to show a parameter-free locally differential private SGD algorithm.
Consider the Laplace sanitization mechanism that adds noise with probability density function $\rho_{\bxi}(\bz) \propto \exp(-\frac{\epsilon}{2} \|\bz\|_2)$. In words, the noise added to the subgradients makes them very similar to one another.
\citet{SongCS15} proved that this mechanism is $\epsilon$-local differentially private. Also, the noise is zero-mean and they proved that $\EE\lt[\|\bxi_t\|_2^2\rt]\leq \frac{4(d^2+d)}{\epsilon^2}$, satisfying \eqref{eq:bounded_variance}.
We now prove that the Laplace mechanism also satisfies the sub-exponential noise assumption \eqref{eq:sub_gamma}. The proof is rather technical, hence we defer it to Appendix~\ref{sec:proof_lemma_private_noise}.
\begin{lemma}
\label{lemma:private_noise}
Let $\bxi \in \R^d$ a random variable drawn from the density $\rho_{\bxi}(\bz) \propto \exp(-\frac{\epsilon}{2}\|\bz\|_2)$. Then
\[
\max_{\|\a\|\le1} \ \E_t\left[\exp(\beta \la \bxi_t, \ba \ra)\right] 
\leq \exp\left(\frac{9 d^2 \beta^2 }{\epsilon^2}\right), \  \forall |\beta| \leq \frac\epsilon4~. 
\]
\end{lemma}
Theorem~\ref{thm:banach-oco} in conjunction with the online-to-batch conversion~\citep{littlestone89online} directly implies the convergence guarantee of a differentially private version of BANCO as stated in the following corollary.
\begin{corollary}
\label{cor:convergence_sgd}
Assume $h(\bw,\bx)$ convex in the first argument $\bw \in \R^d$ and with its subgradients have L2 norm bounded by 1, where the subgradient is with respect to the first argument. Set the uniform prior in BANCO, Algorithm~\ref{alg:oned}, and $\cA_\tD$ being projected OGD with stepsizes $\eta_t=1/\sqrt{\sum_{s=1}^t \|\hbg_s\|_2^2}$ in Algorithm~\ref{alg:reduction} for $T$ iterations on the sequence of losses $\ell_t(\bw)=h(\bw,\bx_t)$, where $\bx_t$ are coming i.i.d. from a distribution $\rho_X$. Set $\hat{\bg}_t=\bg_t+\bxi_t$, where $\bxi_t \in \R^d$ is drawn from the density $\rho_{\bxi}(\bz) \propto \exp(-\frac{\epsilon}{2}\|\bz\|_2)$. Then, for any $\bw^\star \in \R^d$, we have
\[
\E \left[ H\left(\frac1T \sum_{t=1}^T \bw_t\right)\right] - H(\bw^\star)
\leq O\left( \frac{d\|\bw^\star\|_2}{\epsilon \sqrt{T}} \sqrt{\ln\left(1+\tfrac{d^2\|\bw^\star\|_2 T}{\epsilon^2\tau}\right)} +\frac{\tau}{T}\right)~.
\]
\end{corollary}
This convergence rate matches the one for private SGD in \citet{WuLKCJN17} up to polylogarithmic terms, with the important difference that we do not need to assume the knowledge of the norm of the optimal solution $\bw^\star$ to tune the stepsizes.

\subsection{Noisy Coin Betting Implies the Law of Iterated Logarithms in Banach Spaces}
\label{sec:lil}

There is tight connection between concentration inequalities in Banach spaces and online linear optimization algorithms unveiled by~\citet{RakhlinS17}. They showed that online mirror descent with adaptive stepsizes gives rise to self-normalized concentration inequality for martingales.
Hence, it is natural to ask what kind of concentration can be derived from the noisy coin betting algorithms. Here, we show that there is a connection between the law of iterated logarithms for sub-Gaussian RVs in Banach spaces and Algorithm~\ref{alg:reduction}. The exact same reasoning holds for sub-exponential RVs, but we consider the sub-Gaussian case for ease of exposition.

First, consider the one-dimensional case. It is immediate to see that, setting $g_t=0$, we have that \eqref{eq:coinbetting-condition} implies that $\E_t[F_t(\sum_{i=1}^t \hg_i)] \leq F_{t-1}(\sum_{i=1}^{t-1} \hg_i)$, that is $F_t(\sum_{i=1}^t \hg_i)$ is a supermartingale. Hence, we can use Doob's inequality~\cite[Exercise 5.7.1]{durrett10probability} to have
\begin{align}
\label{eq:doob}
\PP\left[\max_{t} F_t\left(\sum_{i=1}^t \hg_i\right) \geq \frac1\delta\right]
\leq \delta \E[F_0(0)] 
= \tau \delta~. 
\end{align}
This inequality allows immediately to derive a concentration inequality. The only missing ingredient is the correct prior on the betting fraction $\beta$ that gives us the optimal bound. We derive it in the following lemma, whose proof is in Appendix~\ref{sec:proof_lem_lil}.
\begin{lemma}
\label{lem:lil}
  Set $\tau=1$ and let $\pi(\beta) = \fr{1}{2\pi |\beta|(\ln^2 (\sig_\oned|\beta|) + 1)}$ be the prior.
  Assume $d=1$, $g_t = 0, \forall t$.
  Let $\xi_t$ be sub-Gaussian (i.e., $b=0$).
  Then,
  \begin{align*}
    F_t\lt(\sum_{s=1}^t \hg_s \rt) \ge  \fr{\exp\lt(\fr{(\sum_{s=1}^t \hg_s)^2}{2t\sig_\oned^2}\rt)}{2\pi\sqrt{e}\fr{\sum_{s=1}^t \hg_s}{\sqrt{t\sig_\oned^2}}\lt(\ln^2\fr{\sum_{s=1}^t \hg_s}{t\sig_\oned} + 1\rt)}~.
  \end{align*}
  Furthermore, the noisy coin betting potential $F_t$ implies 
  \begin{align*}
    \PP\lt(\sup_t \lt|\sum_{s=1}^t \hg_s\rt| \ge \sig_\oned \sqrt{2t\ln\lt(\lt(\fr{6\pi\sqrt{e}}{\dt}\rt)^{3/2}\cdot\lt(\ln^2(\sqrt{t}) + 1\rt)\rt) }\rt) \le \dt~.
  \end{align*}
\end{lemma}
We remark that the choice of prior in Lemma~\ref{lem:lil} resembles $\fr{1}{\beta \ln^2(\beta)}$ used by~\cite{chernov10prediction} and~\cite{koolenVE15}, but their choice does not work when the range of $\beta$ is unbounded.

We now show that the reduction in Algorithm~\ref{alg:reduction} implies a Banach-valued martingale concentration inequality.
Specifically, for the Banach space being $(2,\lam)$-uniformly convex and with the choice of OMD described in Section~\ref{sec:banach} as $\cA_\tD$, we have $\sum_{s=1}^t \la\hbg_s, \bu - \y_s\ra\le \sqrt{\fr{2}{\lam} \sum_{s=1}^t \|\hbg_s\|^2_\star}$ for all $\bu$ in the unit ball $S\subset V$ w.p. 1.
This implies, by the definition of the dual norm,
\begin{align*}
  \lt\| \sum_{s=1}^t \hbg_s \rt\|_\star
  \le  \sqrt{\fr{2}{\lam}  \sum_{s=1}^t \|\hbg_s\|_\star^2} + \sum_{s=1}^t \la\hbg_s, \y_s\ra~.
\end{align*}
Since $\la\hbg_s, \y_s\ra$ is the feedback given to BANCO, Lemma~\ref{lem:lil} implies that
\begin{align*}
  \PP\lt(\max_t \lt\| \sum_{s=1}^t \hbg_s \rt\|_\star
  \ge  \sqrt{\fr{2}{\lam} \sum_{s=1}^t \|\hbg_s\|_\star^2}
     + \sig_\oned \sqrt{2t\ln\lt(\lt(\fr{6\pi\sqrt{e}}{\dt}\rt)^{3/2} \lt(\ln^2(\sqrt{t}) + 1\rt)\rt) }
     \rt) \le \delta~.
\end{align*}

\section{Conclusion and Future Work}
\label{sec:conclusion}

In this paper, we introduced the unconstrained OCO problem with subgradients corrupted by sub-exponential noise, motivated by a recent pessimistic results on learning with unbounded subgradients.
Straightforward extensions of existing algorithms do not result in optimal regret rates.
Hence, we proposed a new algorithm called BANCO via the noisy coin betting framework, which achieves the same optimal minimax regret rate as in the noise-free unconstrained OCO w.r.t. the comparator $\|\bu\|$ and the horizon $T$.
Our lower bound on the noise level $\sig$ implies that the regret bound of BANCO is optimal up to constant factors.
Numerous applications follow naturally including differential privacy, which provides the first parameter-free subgradient descent algorithm for differential privacy.

Our study opens up numerous research directions.
First, one immediate difference in our upper bound from the standard noise-free OCO algorithms is that we do not have a data-dependent regret bound; we have $(G^2 +\sigma^2)T$ rather than $\E[\sum_{t=1}^T \|\hat{\bg}_{t}\|_\star^2]$.
It would be interesting to investigate whether data-dependent bounds are possible.
Second, it would be desirable not to require the knowledge of the noise through $(\sig^2,b)$.
While there are cases where the noise is known ahead of time, such as in private SGD, in the vast majority of applications data arrives through a noisy channel with an unknown noise.
Third, it would be interesting to consider more general noise conditions such as heavy-tailed distributions.
Finally, high probability regret bounds would be a straightforward research direction.

\subsection*{Acknowledgments}

This material is based upon work supported by the National Science Foundation under grant no. 1740762 ``Collaborative Research: TRIPODS Institute for Optimization and Learning.''
We would like to thank Adam Smith for his valuable feedback on differentially-private SGDs.

\bibliographystyle{plainnat}
\bibliography{library-shared,../../../../learning}

\newpage

\appendix

{\centering\section*{Appendices}}

\section{OGD with linear losses}
\label{sec:ogd_olo}

We show that for linear losses OGD's expected regret does not scale with the noise level $\sig$.

Consider the linear losses $\ell_t(\x) = -\la\bg_t, \x\ra$.
Let $G=1$ for simplicity.
Assume that the loss functions are set before the game starts.
That is, $\bg_t$'s are deterministic.
The standard OGD makes predictions by $\w_t = \eta \sum_{s=1}^{t-1} \hbg_s$.
Let $\w^\star_t$ be the prediction that OGD would have made in the noise-free setting: $\w^\star_t = \eta \sum_{s=1}^{t-1} \bg_s$.
It is easy to see that $\EE \w_t = \EE \eta \sum_{s=1}^{t-1} \hbg_s = \eta \sum_{s=1}^{t-1} \bg_s = \w^\star_t$.
Therefore, the expected regret of OGD satisfies
\begin{align*}
  \EE \sum_{t=1}^T \la \hbg_t, \bu - \w_t\ra = \EE \sum_{t=1}^T \la \bg_t, \bu - \w_t\ra = \sum_{t=1}^T \la \bg_t, \bu - \w^\star_t\ra~.
\end{align*}
Therefore, let alone the data-dependent regret, OGD has a regret bound of $O((\|u\|^2+1)\sqrt{T})$ with a tuned $\eta$.
Interestingly, the regret bound does not involve $\sig$.
However, one cannot expect to be free from $\sig$ in general.
Indeed, our lower bound in Theorem~\ref{thm:lb} shows that the factor $\sig$ must be present in general.

\section{Proof of Theorem~\ref{thm:reward-regret}}
\label{sec:proof_thm_reward-regret}

The proof follows from the fact that the expected wealth is underapproximated by the potential function $F(\x)$, together with the definition of the Fenchel conjugacy:
\begin{align}
\E[\hat R^\Lin_T(\bu)] 
&= \E\left[\sum_{t=1}^T \langle \hat{\bg}_t, \bu - \bw_t\rangle\right]
= \E\left[\sum_{t=1}^T \langle \hat{\bg}_t, \bu \rangle - \Wealth_T + \tau\right] \notag
\\ &\le \E\left[\sum_{t=1}^T \langle \hat{\bg}_t, \bu \rangle - F\left(\sum_{t=1}^T \hat{\bg}_t\right) + \tau\right]  \notag
\\ &\leq \E\left[\max_{\bx \in V^\star} \ \langle \bx, \bu \rangle - F\left( \bx \right)+\tau\right] 
= F^\star\left(\bu\right) + \tau~. \label{eq:proof-thm3-1}
\end{align}

\section{Proof of Theorem~\ref{thm:conjugate_uniform_prior}}
\label{sec:proof_thm_conjugate_uniform_prior}

From the definition of the Fenchel duality we have
\begin{align*}
f^\star(u) 
= \max_\theta \ u \theta - f(\theta) 
= u \theta^\star - f(\theta^\star),
\end{align*}
where $\theta^\star = \arg\max_\theta \ u\theta - f(\theta) $.
Define $\beta^\star = \arg \max_\beta \ \exp(\beta\theta^\star - \beta^2 S)$, that is $\beta^\star = \frac{\theta^\star}{2S}$.
Assume that $\theta^\star \ge 0$.
The reasoning is analogous for $\theta^\star <0$.
In fact, one can show that the function is even.

We perform a case-by-case analysis.
We first assume that $\theta^\star \le \sqrt{2S}$.
Then,
\begin{align*}
f^\star(u) 
\le u\theta^\star - f(\theta^\star) 
\le |u|\sqrt{2S},
\end{align*}
from which the stated bound follows.
Hence, we can safely assume $\theta^\star > \sqrt{2S}$, which is equivalent to $\beta^\star \ge \frac{1}{\sqrt{2S}}$.
Let $[v_1,v_2] \subseteq[-a,a \wedge \beta^\star]$.
Recall that we use the uniform prior: $\pi(\beta) = 1/(2a), \forall \beta \in [-a,a]$.
The following inequality becomes useful:
\begin{align} \label{eq:noisy-thm-useful}
f(\theta^\star) 
&= \frac{\tau}{2a} \int_{-a}^a \exp(\beta\theta^\star - \beta^2S)\diff \beta
\ge \frac{\tau}{2a}\int_{v_1}^{v_2} \exp(\beta\theta^\star -\beta^2S) \diff\beta \\
&\ge \frac{v_2-v_1}{2a} \tau\exp(v_1\theta^\star - v_1^2S)~.
\end{align}

\noindent\textbf{Case 1}: $\beta^\star \le a$.\\
Using~\eqref{eq:noisy-thm-useful} with $v_1 = \beta^\star - \frac{1}{\sqrt{2S}}$ and $v_2 = \beta^\star$, we have
\begin{align*}
f(\theta^\star) 
= \frac{\tau}{2a}\int_{-a}^a \exp(\beta\theta^\star - \beta^2S) \diff\beta
\ge \frac{\tau}{2a\sqrt{2S}}\exp\left(\frac{(\theta^\star)^2}{4S}-\frac{1}{2}\right)~.
\end{align*}
Hence, we have
\begin{align*}
f^\star(u) 
&\le |u||\theta^\star| - \frac{\tau}{2a\sqrt{2S}}\exp\left(\frac{(\theta^\star)^2}{4S} - \frac{1}{2}\right) \\
&\le \max_x \ x|u| - \frac{\tau}{2a\sqrt{2S}} \exp\left(\frac{x^2}{4S} - \frac{1}{2}\right)~.
\end{align*}
To solve the problem above, we consider the following stylized problem:
\begin{align*}
\max_x x|u| - A \exp(Bx^2 - C) \;.
\end{align*}
We see by setting the gradient to zero that $A(2Bx)\exp(Bx^2 - C) = |u| \implies 4 A^2 B^2 x^2 \exp(2Bx^2 - 2C) = u^2 $.
Letting $z=2Bx^2$ and $D = \frac{u^2}{2A^2B e^{-2C}}$, we have $z\exp(z) = D$.
Using Lambert function, we have $z = W_0(D)$ and so $x = \sqrt{\frac{W_0(D)}{2B}}$, which we call $x^\star$.
We use the upper bound on $W_0(y) \le \ln(1 + y)$ for $y >0$ by \citet[Lemma 17]{orabona16coin}.
Then, plugging in $A = \frac{\tau}{2a\sqrt{2S}}$, $B = 1/(4S)$, and $C = 1/2$,
\begin{align*}
f^\star(u) \le x^\star|u| \le |u|\sqrt{2S\ln\left(1 + \frac{16ea^2S^2u^2}{\tau^2}\right)}~.
\end{align*}

\noindent\textbf{Case 2}: $\beta^\star>a$.\\
In this case, we have $\theta^\star > 2Sa$.
Then, choose $v_1 =a-c$  and $v_2=a$ to arrive at
\begin{align*}
f(\theta^\star) 
&= \frac{1}{2a}\int_{-a}^a \tau\exp(\beta\theta^\star - \beta^2S)\diff\beta
\ge \frac{v_2-v_1}{2a} \tau\exp(v_1\theta^\star - v_1^2S) \\
&\ge \frac{c}{2a} \tau\exp\left(v_1\theta^\star - v_1^2\frac{\theta^\star}{2a}\right) 
= \frac{c}{2a} \tau\exp( \theta^\star Q),
\end{align*}
where $Q = v_1 - \frac{v_1^2}{2a}$.
Using $\theta^\star>0$,
\begin{align*}
f^\star(u) 
\le u \theta^\star - \frac{c}{2a}\tau \exp(\theta^\star Q) 
\le \max_{\theta} \  |u| \theta - \frac{c}{2a}\tau \exp(\theta Q) 
= \frac{|u|}{Q} \ln\left(|u|\frac{2a}{e c \tau Q }\right)~.
\end{align*}
Setting $c=a/2$, we have $Q = \frac{3}{8}a$, which leads to $f^\star(u) = \frac{8}{3a} |u|\ln\left(\frac{32}{3 e a\tau}|u|\right)$.

\section{Proof of Theorem~\ref{thm:banach-oco}}
\label{sec:proof_thm_banach-oco}

Observe that $|s_t| \le \|\hbg_t\|_\star \|\by_t\| \le \|\hbg_t\|_\star$ since $\|\by_t\| \le 1$ for all $t$.
Furthermore,
\begin{align*}
\EE\Regret_T(\bu)
&\le \EE\lt[\sum_{t=1}^T \langle \g_t, \bu-\bx_t\rangle \rt]
=\EE \lt[\sum_{t=1}^T \langle \hbg_t, \bu-\bx_t\rangle \rt] \\
&=\EE \lt[ \sum_{t=1}^T \langle \hbg_t, \bu\rangle - \langle \hbg_t, w_t\by_t\rangle \rt]\\
&=\EE \lt[\sum_{t=1}^T  \langle \hbg_t, \bu\rangle - \langle \hbg_t, \by_t\rangle \|\bu\| + \langle \hbg_t, \by_t\rangle\|\bu\| -\langle \hbg_t, \by_t\rangle w_t \rt]\\
&\le \|\bu\| \EE \lt[\sum_{t=1}^T \langle \hbg_t,\bu/\|\bu\|\rangle - \langle \hbg_t, \by_t\rangle \rt] + R^\tM_T(\|\bu\|)\\
&\le \|\bu\|R^\tD_T(\bu/\|\bu\|) + R^\tM_T(\|\bu\|)~.
\end{align*}

\section{Proof of Theorem~\ref{thm:lb}}
\label{sec:proof_thm_lb}

It is not hard to see that a stochastic optimization lower bound imply an online learning lower bounds.
This is due to the online to batch conversion~\citep{littlestone89online} which implies stochastic optimization is ``not harder'' than online learning.
Specifically, suppose we have a lower bound on the convergence of stochastic optimization for convex functions: $\E F(\bx_T) - F(\bx^\star) \ge c / \sqrt{T}$. 
Then, we can claim a lower bound in the online convex optimization: $\E \sum_{t=1}^T f_t(\bx'_t) - f_t(\bu) \ge c / \sqrt{T}$.
Here is a proof: 
Suppose a better rate is possible in online learning with some method: $\E \sum_{t=1}^T f_t(\bx'_t) - f_t(\bu) < c / \sqrt{T}$.
One can then perform online learning with $f_t = F$ where the online learner acquires noisy version $\hat f_t$.
With the online-to-batch conversion, this solves the stochastic optimization with a better rate: $\EE F(\fr{1}{T}\sum_{t=1}^T \bx'_t) - F(\bx^\star) < c / \sqrt{T}$, which is a contradiction.

Therefore, it suffices to show a lower bound on stochastic optimization.
Before presenting the lower bound statement, we describe the problem setup.
We closely follow the setup of~\citet{agarwal12information}.
Let $\SS\subseteq\R^d$.
Let the function class $\cF$ consists of functions $f: \SS\rarrow\R$ that are convex and 1-Lipschitz w.r.t. $\ell_p$-norm: $|f(\bx) - f(\by)| \le L\| \bx - \by \|_p, \forall \bx,\by\in \SS$.
An algorithm $\blue{\cM}$ has access to $T$ calls of the first order oracle and outputs $\blue{\bx_T}$ after $T$ calls to the oracle (hereafter, we color definitions with light blue for the benefit of readers).
The oracle $\psi_{\sig}(\bx,f)$ takes $\bx\in \SS$ and returns $(\blue{\hat f(\bx)}, \blue{\hat z(\bx)})$ where $\hat f(\bx)$ is the noisy function value  and $\hat z(\bx)$ is a noisy subgradient such that $\EE \hat z(\bx) \in \partial f(\bx)$.
The oracle guarantees a noise condition $\EE \|\hat z(\bx) - \EE \hat z(\bx) \|_q^2 \le \blue{\sig^2}$.
Our goal is to  find a lower bound on 
\[ \blue{\eps^\star} := \inf_{\cM} \sup_{f \in \cF} \EE\lt[ f(\bx_T) - f(\bx_f^\star) \rt], \]
where $\bx_f^\star$ is the minimizer of $f$.
The quantity $\epsilon^\star$ depends on $d$, $T$,$\sig$, and $\SS$.

Let $\BB_\infty(r)$ be the $\infty$-norm ball with radius $r$. 
We present our stochastic optimization lower bound in Theorem~\ref{thm:stochastic-lb-noise} below.
The difference from the lower bound in~\citet{agarwal12information} is that the bound therein is that $(i)$ they obscure the dependence on the noise $\sig$ by equating it to the Lipschitz constant and $(ii)$ they assume uncentered second moment noise bound  $\EE \| \hat z(\bx) \|_q^2 \le \sig^2$ rather than the variance of $\|\hat z(\bx) \|_q$.
Departing from the prior work, we consider a different function class that keeps the Lipschitz constant at 1 while allowing the noise level $\sig$ to be arbitrarily large.
\begin{theorem}\label{thm:stochastic-lb-noise}
  Let $r$ be the largest number such that $\BB_\infty(r) \subseteq \SS$.
  Let $\sig \ge 2$.
  Then, there exists a universal constant $c_0$ such that
  \begin{align*}
  q\in[1,2] &\implies \eps^\star \ge \min\lt\{c_0 \sig \fr{r\sqrt{d}}{\sqrt{T}} , \fr{r}{18}\rt\} \quad\text{ and } \\
  q \ge 2 &\implies \eps^\star \ge \min\lt\{c_0 \sig \fr{r d^{1/q - 1}}{\sqrt{T}}, \fr{r d^{1/q - 1}}{18}\rt\} ~.
  \end{align*}
\end{theorem}
\begin{proof}
  The proof closely follows~\citet{agarwal12information}, but we consider a different function class.
  The key idea is to construct a function class such that identification of the target function is equivalent to identification of coefficients $\{\alpha_i \in [0,1]\}, i\in\{1,\ldots,d\},$ on a set of basis functions.
  Furthermore, the construction defines an oracle such that each query amounts to revealing a coin outcome $\{0,1\} \sim \text{Bernoulli}(\alpha_i)$ for some $i$'s (details vary for different $q$'s). 
  Then, the number of observations in statistical estimation is directly connected to the number of oracle calls, allowing a statistical lower bound to imply an iteration complexity of stochastic optimization.
  
  Let $\blue{\cV} \subseteq \{\pm1\}^d$ has $\blue{M}$ distinct vertices of $d$-dimensional hypercube such that $(i)$ $\cV$ is $\fr{d}{4}$-packing w.r.t. hamming distance (i.e., $\sum_{i} \one\{\alpha_i \neq \beta_i\} \ge \fr{d}{4}, \forall \balpha \neq \bbeta \in \cV$) and $(ii)$ $M \ge (2/\sqrt{e})^{d/2} \approx 1.1^d$.
  Such a packing is known to be possible~\citep{matouvsek02lectures}.
  We define the function class $\cG(\dt)$ that consists of
  \begin{align*}
  &\forall \balpha \in \cV, \quad \blue{g_\balpha(\bx)} := \fr{c}{d} \sum_i \lt( \lt(\fr{1}{2} + \alpha_i\dt\rt) f^+_i(\bx) + \lt(\fr{1}{2}-\alpha_i\dt\rt)  f^-_i(\bx) \rt) \quad \text{ where }\\
  \blue{f^+_i(\bx)} &:= \begin{cases}
  -x_i + r(-\sig-1) & \text{ if } x_i \le -r
  \\ \sig x_i & \text{ if } -r\le x_i \le r
  \\ x_i + r(\sig-1) & \text{ if } r \le x_i
  \end{cases}
  \quad\text{ and }\quad
  \blue{f^-_i(\bx)} := \begin{cases}
  -x_i + r(\sig-1) & \text{ if } x_i \le -r
  \\ -\sig x_i & \text{ if } -r\le x_i \le r
  \\ x_i + r(-\sig-1) & \text{ if } r \le x_i
  \end{cases}~.
  \end{align*}
  We assume that $\dt \le \fr{1}{2\sig}$, which ensures the convexity of $g_\balpha$.
  
  \noindent\textbf{Case 1}: $q\in[1,2]$.  \\
  For this case, we assume an oracle that first chooses $\blue I \in [d]$ uniformly at random, draw $\blue {b_I} \in \{0,1\}$ with $\text{Ber}(1/2 + \alpha_I \dt)$, and then return the function value and the subgradient of
  \begin{align}\label{eq:agarwal09-oracle}
  \blue{\hat g_{\balpha}(\bx)} = c \lt( b_I f_I^+(\bx) + (1-b_I) f_I^-(\bx)\rt)~.
  \end{align}
  Thus, the learner only sees either $c f_I^+(\bx)$ or $c f_I^-(\bx)$, and the function value and the subgradient are unbiased.
  Denote by $\blue{\hat z_\balpha(\bx)}$ be the noisy subgradient returned by the oracle such that $\blue{z_\balpha(\bx)} := \EE[\hat z_\balpha(\bx)] \in \partial g_\balpha(\bx)$.
  
  Some facts on the subgradient norms:
  \begin{itemize}
    \item $\| z_\balpha(\bx)\|_q^2 \le \max\{  \fr{c^2}{d^2}, \fr{4c^2\dt^2\sig^2}{d^2} \} \|\mathbf{1}\|_q^2 = c^2 d^{(2/q) - 2}$.
    \item $\EE \| \hat z_\balpha(\bx)\|_q^2 \le c^2 \sig^2$.
    \item $\EE \| \hat z_\balpha(\bx) - z_\balpha(\bx)\|_q^2 \le 2\EE [\| \hat z_\balpha(\bx)\|_q^2] + 2\| z_\balpha(\bx)\|_q^2 \le 2(c^2 \sig^2 + c^2 d^{(2/q) - 2})$.
  \end{itemize}
  By setting $c=1/2$, $g_\balpha(\bx)$ is 1-Lipschitz and the noise variance is bounded: $\EE \| \hat z_\balpha(\bx) - z_\balpha(\bx) \|_q^2 \le \sig^2$.
  
  We define a premetric $\rho$:
  \begin{align*}
  \blue{\rho(f,g)} := \inf_{x\in\SS} f(\bx) + g(\bx) - f(\bx_f^\star) - g(\bx_g^\star) 
  \end{align*}
  which is $0$ if and only if $\bx_f^\star = \bx_g^\star$ (assuming $f$ and $g$ have a unique minimizer).
  Define $\blue{\psi(\dt)} : = \min_{\balpha \neq \bbeta \in \cV} \rho(g_\balpha, g_\bbeta)$.
  We study $\rho(g_\balpha, g_\bbeta)$ where $\balpha, \bbeta \in \cV$ such that $\balpha \neq \bbeta$.
  By examining the function carefully, one can show that $\rho(g_\balpha, g_\bbeta) \ge \fr{c}{d} (\sum_i \one\{\alpha_i \neq \beta_i\})4\dt r\sig$.
  Since $\sum_i \one\{\alpha_i \neq \beta_i\} \ge \fr{d}{4}, \forall \balpha\neq\bbeta \in \cV$, we have
    \begin{align}\label{eq:lb-proof-1}
    \psi(\dt) = \min_{\balpha\neq\bbeta} \rho(g_\balpha,g_\bbeta) \ge c\dt r\sig \ .
    \end{align}
  
  Now, the main argument is as follows.
  If $\eps^\star \ge \fr{cr}{18}$, then we have the half of the theorem statement.
  Therefore, it suffices to consider the regime $\eps^\star < \fr{cr}{18}$.
  
  In this regime, we consider the function class $\cG(\dt)$ with $\dt = \fr{9}{cr\sig} \eps^\star$.
  This implies that $(i)$ $\dt \le \fr{9}{cr\sig} \fr{cr}{18} = \fr{1}{2\sig} \le \fr{1}{4}$ and that $(ii)$ there exists a method $\blue{\cM^\star}$ such that $\sup_{f \in \cG(\dt)} \EE[f(\bx_T) - f(\bx_f^\star)] \le \eps^\star = \fr{c\dt r \sig}{9} \le \psi(\dt)/9$ by the definition of $\eps^\star$ and~\eqref{eq:lb-proof-1}.
  
  By~\citet[Lemma 2]{agarwal12information}, these two conditions, $\dt \le 1/4$ and $\sup_{f \in \cG(\dt)} \EE[f(\bx_T) - f(\bx_f^\star)] \le \psi(\dt)/9$, imply the following: For any $\balpha^\star\in \cV$, facing to solve the optimization problem with the function $g_{\balpha^\star}$, one can invoke $\cM^\star$ to construct an estimator $\hat \balpha \in \cV$ of the true $\balpha^\star$:
  \begin{align*}
  \forall \balpha^\star\in \cV , \PP(\hat\balpha \neq \balpha^\star) \le 1/3~.
  \end{align*}
  On the other hand,~\citet[Lemma 3]{agarwal12information} use Fano's inequality to show that
  \begin{align*}
  \PP(\hat\balpha \neq \balpha^\star) \ge 1 - 2 \fr{16T\dt^2 + \ln 2}{d\ln(2/\sqrt{e})}~.
  \end{align*}
  Combining these two results, we have $1 - 2 \fr{16T\dt^2 + \ln 2}{d\ln(2/\sqrt{e})} \le \fr{1}{3}$
  Using $\dt = \fr{9}{cr\sig} \eps^\star$, one can show that, for $d\ge 11$,
  \begin{align*}
  \eps^\star = \Omega\lt( c\fr{\sig r \sqrt{d}}{\sqrt{T}} \rt)~.
  \end{align*}
  For $d \le 10$, simply consider a reduction to $d=1$ case and use the Le Cam's bound~\cite[Lemma 4]{agarwal12information}. 
  This completes the first part of the proof.
  
  \noindent\textbf{Case 2:} $q \ge 2$.\\
  For the second part, we consider a different oracle that chooses $d$ independent coin flips $b_i \sim \text{Bernoulli}(\frac{1}{2} + \alpha_i \dt)$, $i\in\{1,\ldots,d\}$, and return the function value and the subgradient of
  \begin{align}\label{eq:agarwal09-oracle-B}
  \blue{\hg_\balpha(\bx)} = \fr{c}{d} \sum_i \lt( b_i f_i^+(\bx) + (1-b_i) f_i^-(\bx)\rt)~.
  \end{align}
  This provides unbiased function values and subgradients, and corresponds to revealing one coin outcome for each dimension.
  While this provides more information for the coin tossing (easier problem), but it allows steeper per-coordinate subgradients than the oracle A (harder problem), given the same Lipschitz constants.
  
  The difference of the proof is just on the subgradient norms and how we set $c$.
  Recall that $\| z_\balpha(\bx)\|_q^2 \le \max\{  \fr{c^2}{d^2}, \fr{4c^2\dt^2\sig^2}{d^2} \} \|\mathbf{1}\|_q^2 = c^2 d^{(2/q) - 2}$. 
  One can see that $\EE \| \hat z_\balpha(\bx)\|_q^2 = \fr{c^2}{d^2} \sig^2 \|\bfone\|_q^2 = c^2\sig^2 d^{\fr{2}{q} - 2}$.
  Then, the subgradient noise variance is bounded: 
  \begin{align*}
  \EE \| \hz_\balpha(\bx) - z_\balpha(\bx) \|_q^2 \le 2\EE [\|\hz_\balpha(\bx)\|_q^2] + 2\| z_\alpha(\bx)\|_q^2 \le 2(c^2 \sig^2 d^{\fr{2}{q} - 2} + c^2 d^{\fr{2}{q} - 2}) \le 4c^2 \sig^2 d^{\fr{2}{q} - 2}~.
  \end{align*}
  By setting $c = \frac{1}{2}d^{-(\fr{1}{q} - 1)}$, we satisfy 1-Lipschitz ($\| z_\balpha(\bx)\|_q \le 1$) and the noise level controlled: $\EE \| \hz_\balpha(\bx) -  z_\balpha(\bx) \|_q^2 \le \sig^2$.
  
  Again, the oracle here is equivalent to discovering all the $d$ coin outcomes in each iteration rather than one.
  By~\citet[Lemma 3]{agarwal12information} with $\ell=d$, we have that $\PP(\hat\balpha \neq \balpha^\star) \ge 1 - 2\fr{16Td\dt^2 + \ln 2}{d\ln(2/\sqrt{e})}$.
  With the same logic, we have $1 - 2 \fr{16Td\dt^2 + \ln 2}{d\ln(2/\sqrt{e})} \le  \fr{1}{3}$.
  Again, by $\dt = \fr{9}{cr\sig}\eps^\star$, one can show that, for $d\ge 11$,
  \begin{align*}
  \eps^\star = \Om\lt(c \fr{\sig r}{\sqrt{T}}\rt) = \Om\lt(d^{1-1/q} \fr{\sig r}{\sqrt{T}}\rt)~.
  \end{align*}
  For $d\le 10$, the same argument as the case 1 can be made.
\end{proof}

To prove Theorem~\ref{thm:lb}, it suffices to notice that the largest $r$ such that $\BB_\infty(r) \subseteq \SS$ with $\SS$ being the $\ell_p$-norm ball of radius $U$ is $r = Ud^{-1/p}$.

\section{Proof of Lemma~\ref{lemma:private_noise}}
\label{sec:proof_lemma_private_noise}

The Laplace mechanism noise can be obtained by multiplying independent random variables $\bz$ and $m$, where $\bz$ is a drawn uniformly over the L2 ball, and $m$ is an Erlang distribution with shape equal to $d$ and rate $\frac{\epsilon}{2}$~\citep{WuLKCJN17}.
This implies that
\[
\E_{\xi}[\exp(\beta \langle\xi,\ba\rangle] 
= \E_{\bz,m}\left[\exp\left(\beta m \langle \bz,\ba\rangle\right)\right]
= \E_{\alpha,m}[\exp(\beta m \alpha)]~.
\]
where $\alpha$ is a random variable that model the cosine of the angles between $\bz$ $\ba$. In the one-dimensional case, it is easy to see that $\alpha$ is a Rademacher variable. Hence, we have
\[
\E_{\xi}[\exp(\beta \langle\xi,\ba\rangle] 
= \frac{1}{2}\E_{m}[\exp(\beta m)+\exp(-\beta m)]~.
\]
Instead, for $d\geq2$, we the calculation is more involved, but we show that we still get the same result. In particular, observing that $\langle\bz,\ba\rangle$ is the cosine of random angles distributed uniformly between $-\pi$ and $\pi$, we have that $\alpha$ is drawn from the distribution $\rho_\alpha(x) = \frac{1}{\pi\sqrt{1-x^2}}$.
The expectation $\E_{\alpha}[\exp(\beta m \alpha)]$ can be computed in a closed form, being equal to modified Bessel function of the first kind $I_0(\beta m)$.
From \citet[Formula 6.25]{Luke72}, we use the inequality
\[
\Gamma(\nu+1)\left(\frac{2}{x}\right)^\nu I_\nu(x) < \frac12 \left(\exp(x)+\exp(-x)\right),  \ \forall x>0, \nu>-\frac{1}{2},
\]
that implies
\[
\E_{\alpha,m}[\exp\left(\beta m \alpha\right)]
< \frac12\E_{m}\left[\exp(\beta m)+\exp(-\beta m)\right], 
\]
as in the one-dimensional case.

Hence, taking the expectation with respect to $m$ and using the formula for the moment generating function of the Erlang distribution, we get
\begin{align*}
\E_{\alpha,m}[\exp(\beta m \alpha)]
&< \frac12 \left[\left(1-\frac{2\beta}{\epsilon}\right)^{-d}+\left(1+\frac{2\beta}{\epsilon}\right)^{-d}\right]\\
&=\frac12 \left[\exp\left(d \ln \frac{1}{1-\frac{2\beta}{\epsilon}}\right)+\exp\left(d \ln \frac{1}{1+\frac{2\beta}{\epsilon}}\right)\right] \\
&=\frac12 \left[\exp\left(d \ln \left(1+\frac{2\beta}{\epsilon-2\beta}\right)\right)+\exp\left(d \ln \left(1-\frac{2\beta}{\epsilon+2\beta}\right)\right)\right]\\
&\leq\frac12 \left[\exp\left(d \frac{2\beta}{\epsilon-2\beta}\right)+\exp\left(d \frac{2\beta}{\epsilon+2\beta}\right)\right],
\end{align*}
where in the last inequality we used the elementary $\ln(1+x)\leq x, \ \forall x>-1$.
We now observe that
\begin{align*}
\frac12 \left[\exp\left(d \frac{2\beta}{\epsilon-2\beta}\right)+\exp\left(d \frac{2\beta}{\epsilon+2\beta}\right)\right]
&=\frac12 \exp\frac{d \left(\frac{2\beta}{\epsilon}\right)^2 }{1-\left(\frac{2\beta}{\epsilon}\right)^2}
\left[\exp\frac{d \frac{2\beta}{\epsilon} }{1-\left(\frac{2\beta}{\epsilon}\right)^2}
+\exp\frac{-d \frac{2\beta}{\epsilon} }{1-\left(\frac{2\beta}{\epsilon}\right)^2}
\right] \\
&\leq \exp\frac{d \left(\frac{2\beta}{\epsilon}\right)^2 }{1-\left(\frac{2\beta}{\epsilon}\right)^2}
\exp\frac{d^2 \frac{2\beta^2}{\epsilon^2} }{\left(1-\left(\frac{2\beta}{\epsilon}\right)^2\right)^2},
\end{align*}
where we used the elementary inequality $\exp(x)+\exp(-x) \leq 2 \exp(x^2/2), \ \forall x$.
Overapproximating and using the assumption on $\beta$, we have the stated bound.

\section{Proof of Lemma~\ref{lem:lil}}
\label{sec:proof_lem_lil}

\begin{proof}
  It suffices to consider $\sig_\oned=1$ since the result for $\sig_\oned\neq1$ can be obtained by replacing $S$ below with $\sum_{s=1}^t \hg_s / \sig_\oned$.
  Let $S = \sum_{s=1}^t \hg_s$.
  Define $\beta^\star = S/t$ and $u = \beta^\star - \fr1{\sqrt{t}}$.
  Then, $\exp(\beta S - \beta^2 t / 2)$ is maximized at $\beta=\beta^\star$ and increasing in $[u,\beta^\star]$.
  Recall that $F_t\lt(\sum_{s=1}^t \hg_s \rt)  = \int_{-\infty}^{\infty} \pi(\beta) \exp\lt(\beta S - \fr{\beta^2 t}{2}\rt) \diff\beta$.
  To evaluate the integral, it suffices to assume $S\ge0$ since the integrand is symmetric.
  Using the fact that the prior is nonincreasing in $(0,\infty)$,
  \begin{align*}
  F_t\lt(\sum_{s=1}^t \hg_s \rt)
  &\ge \fr{1}{2\pi} \int_u^{\beta^\star} \fr{1}{\beta^\star(\ln^2\beta^\star + 1)} \exp(uS - u^2 t / 2) \diff\beta
  \\&= \fr{1}{2\pi}  \fr{1/\sqrt{t}}{\beta^\star(\ln^2\beta^\star + 1)} \cdot \exp(uS - u^2 t / 2) 
  \\&= \fr{1}{2\pi}  \fr{1}{\fr{S}{\sqrt{t}}\lt(\ln^2\fr{S}{t} + 1\rt)}  \cdot \exp\lt(\fr{S^2}{2t} -\fr12\rt)~.
  \end{align*}
  By~\eqref{eq:doob},
  \begin{align*}
  \PP\lt(\max_t \fr{1}{2\pi\sqrt{e}}  \fr{1}{\fr{S}{\sqrt{t}}\lt(\ln^2\fr{S}{t} + 1\rt)}  \cdot \exp\lt(\fr{S^2}{2t}\rt) \ge \fr1\dt\rt) \le \PP\lt(\max_t F_t(S) \ge \fr1\dt\rt)  \le \dt~.
  \end{align*}
  Rearranging the inequality in the LHS above, we have
  \begin{align*}
  \max_t S^2 \ge 2t\ln\lt(\fr{2\pi\sqrt{e}}{\dt}\cdot\fr{S}{\sqrt t}\lt(\ln^2\fr{S}{t} + 1\rt)\rt)~.
  \end{align*}
  To complete the proof, it suffices to find a tighter and simpler inequality.
  This is equivalent to assuming $S^2 \le \text{[the RHS above]}$ and deriving an upper bound on $S^2$, then inverting it.
  Therefore, it suffices to show
  \begin{align}\label{eq:lem-concentration-1}
  S^2 < 2t\ln\lt(\fr{2\pi\sqrt{e}}{\dt}\cdot\fr{S}{\sqrt t}\lt(\ln^2\fr{S}{t} + 1\rt)\rt)
  \implies 
  S^2 < 2t\ln\lt(\lt(\fr{6\pi\sqrt{e}}{\dt}\rt)^{3/2}\cdot\lt(\ln^2(\sqrt{t}) + 1\rt)\rt)~.
  \end{align}
  
  Let $A = 2 \pi\sqrt{e}/\dt$.
  Using $\ln^2(x) + 1 \le x, \forall x\ge1,$ and $x \le (1/2)\ln x, \forall x>0,$ 
  \begin{align*}
  &&S^2 
  &< 2t\ln\lt(A\cdot\fr{S}{\sqrt t}\lt(\ln^2\fr{S}{t} + 1\rt)\rt)
  \\&&&\le 2t\ln\lt(A\cdot\fr{S^2}{t^{3/2}}\rt)
  = 4t\ln\lt( \sqrt{A}\cdot\fr{S}{t^{3/4}}\rt)
  \\&&&\le 2t\cdot \sqrt{A}\cdot\fr{S}{t^{3/4}}
  \\ &\implies& S &\le 2t^{1/4} \sqrt{A}
  \\ &\stackrel{(a)}{\implies}&
  S^2  &< 2t\ln\lt(2\fr{A^{3/2}}{t^{1/4}}\lt(\ln^2\fr{S}{t} + 1\rt)\rt),
  \end{align*}
  where $(a)$ is by the first inequality.
  
  It suffices to assume the regime $S^2 > t$ since $S^2 \le t$ trivially implies the RHS of \eqref{eq:lem-concentration-1}.
  Since $\ln^2 x$ is decreasing up to 1 and then increasing, we perform a case by case analysis.
  
  \noindent\textbf{Case 1}: $S\le t$.\\
  Since $\ln^2(S/t) = \ln^2(t/S)$ and $t/S\ge1$, we need to upper-bound $t/S$.
  Using $S^2>t$, we have $\ln^2(t/S) \le \ln^2(\sqrt{t})$, which implies the RHS of \eqref{eq:lem-concentration-1}.
  
  \noindent\textbf{Case 2}: $S> t$.\\
  With a similar derivation as above, we have $S^2 < 6t \ln(A^{1/3} S^{2/3} t^{-1/2}) \le 3 t^{1/2} A^{1/3} S^{2/3}$, which implies $S < 3^{3/4} t^{3/8} A^{1/4}$.
  Then,
  \begin{align*}
  S^2 
  < 2t\ln\lt(A\cdot\fr{S^2}{t^{3/2}}\rt)
  \le 2t\ln\lt(A\cdot\fr{3^{3/2}t^{3/4} A^{1/2}}{t^{3/2}}\rt)
  \le 2t\ln\lt(\fr{(3A)^{3/2}}{t^{3/4}} \rt),
  \end{align*}
  which implies the RHS of~\eqref{eq:lem-concentration-1}.
\end{proof}

\end{document}